\documentclass[twoside,11pt]{article}

\usepackage{amsmath,bm,paralist}
\usepackage{jmlr2eArxiv}
\usepackage{algorithm}
\usepackage{algorithmic}

\DeclareMathOperator*{\argmax}{argmax}

\DeclareMathOperator*{\st}{s.t.}

\def \R {\mathbb{R}}
\def \x {\mathbf{x}}
\def \E {\mathrm{E}}

\def \F {\mathcal{F}}

\def \B {\mathcal{B}}

\def \y {\mathbf{y}}

\def \w {\mathbf{w}}
\def \wh {\widehat{\w}}

\def \D {\mathcal{D}}

\def \C {\mathcal{C}}
\def \B {\mathcal{B}}

\def \W {\mathcal{W}}

\def \g {\mathbf{g}}

\def \fb {\bar{f}}

\def \O {\widetilde{O}}

\newtheorem{thm}{Theorem}

\newtheorem{cor}[thm]{Corollary}

\jmlrheading{1}{2015}{1-48}{9/15}{}{Lijun Zhang, Tianbao Yang, Rong Jin, and Zhi-Hua Zhou}

\ShortHeadings{Online Stochastic Linear Optimization under One-bit Feedback}{Zhang, Yang, Jin, and Zhou}
\firstpageno{1}

\begin{document}

\title{Online Stochastic Linear Optimization under \\ One-bit Feedback}

\author{\name Lijun Zhang \email zhanglj@lamda.nju.edu.cn\\
       \addr National Key Laboratory for Novel Software Technology\\
       Nanjing University, Nanjing 210023, China
       \AND
       \name Tianbao Yang \email tianbao-yang@uiowa.edu\\
       \addr Department of Computer Science\\
       the University of Iowa, Iowa City, IA 52242, USA
       \AND
       \name Rong Jin \email rongjin@cse.msu.edu\\
       \addr Department of Computer Science and Engineering\\
        Michigan State University, East Lansing, MI 48824, USA
       \AND
       \name Zhi-Hua Zhou \email zhouzh@lamda.nju.edu.cn\\
       \addr National Key Laboratory for Novel Software Technology\\
       Nanjing University, Nanjing 210023, China}

\editor{Leslie Pack Kaelbling}

\maketitle

\begin{abstract}
In this paper, we study a special bandit setting of online stochastic linear optimization, where only one-bit of information is revealed to the learner at each round. This problem has found many applications including online advertisement and online recommendation.  We assume the binary feedback is a random variable generated from the \emph{logit} model, and aim to minimize the regret defined by the unknown linear function. Although the existing method for generalized linear bandit can be applied to our problem, the high computational cost makes it impractical for real-world problems. To address this challenge, we develop an efficient online learning algorithm by exploiting particular structures of the observation model. Specifically, we adopt online Newton step to estimate the unknown parameter and derive a tight confidence region based on the exponential concavity of the logistic loss. Our analysis shows that the proposed algorithm achieves a regret bound of $\O(d\sqrt{T})$, which matches the optimal result of stochastic linear bandits.
\end{abstract}


\section{Introduction}
Online learning with bandit feedback plays an important role in several industrial domains, such as ad placement, website optimization, and packet routing~\citep{Bandit:suvery}. A canonical framework for studying this problem is the multi-armed bandits (MAB), which models the situation that a gambler must choose which of $K$ slot machines to play~\citep{Robbins:52}. In the basic stochastic MAB, each arm is assumed to deliver rewards that are drawn from a fixed but unknown distribution.  The goal of the gambler is to minimize the regret, namely the difference between his expected cumulative reward and that of the best single arm in hindsight~\citep{Auer:2002:FAM}.
Although MAB is a powerful framework for modeling online decision problems, it becomes intractable when the number of arms is very large or even infinite. To address this challenge, various algorithms have been designed to exploit different structure properties of the reward function, such as Lipschitz~\citep{Kleinberg:2008:MBM} and convex~\citep{Flaxman:2005:OCO,Stochastic_Bandit}. Among them, stochastic linear bandits (SLB) has received considerable attentions during the past decade~\citep{Auer:UCB,Linear:Bandit:08,NIPS2011:LSB}. In each round of SLB, the learner is asked to choose an action $\x_t$ from a decision set $\D \in \R^d$, then he observes $y_t$ such that
\begin{equation} \label{eqn:obr:1}
\E[y_t|\x_t]= \x_t^\top  \w_*,
\end{equation}
where $\w_* \in\R^d$ is a vector of unknown parameters. The goal of learner is to minimize the (pseudo) regret
\begin{equation} \label{eqn:reg}
T \max_{\x \in \D} \x^\top\w_* - \sum_{t=1}^T \x_t^\top \w_*.
\end{equation}

In this paper, we consider a special bandit setting of online linear optimization where the feedback $y_t$ only contains one-bit of information. In particular, $y_t \in \{\pm 1\}$. Our setting is motivated from the fact that in many real-world applications, such as online advertising and recommender systems, user feedback (e.g., click or not) is usually binary. Since the feedback is binary-valued, we assume it is generated according to the logit model~\citep{The-Elements-of-Statistical-Learning-2009}, i.e.,
\begin{equation} \label{eqn:obr:2}
\Pr[y_t= \pm 1|\x_t] = \frac{1}{1+\exp(-y_t \x_t^\top \w_*)}=\frac{\exp(y_t \x_t^\top \w_*)}{1+\exp(y_t \x_t^\top \w_*)}.
\end{equation}
Without loss of generality, suppose $1$ is the preferred outcome. Then, it is natural to define the regret in terms of the expected times that $1$ is observed, i.e.,
\begin{equation} \label{eqn:log:reg}
T \max_{\x \in \D} \frac{\exp(\x^\top \w_*)}{1+\exp(\x^\top \w_*)} - \sum_{t=1}^T \frac{\exp( \x_t^\top \w_*)}{1+\exp( \x_t^\top \w_*)}.
\end{equation}
The observation model in (\ref{eqn:obr:2}) and the nonlinear regret in (\ref{eqn:log:reg}) can be treated as a special case of the Generalized Linear Bandit (GLB)~\citep{NIPS2010}. However, the existing algorithm for GLB is inefficient in the sense that: i) it is not a truly online algorithm since the whole learning history is stored in memory and used to estimate $\w_*$; and ii) it is limited to the case that the number of arms is finite because an upper bound for each arm needs to be calculated explicitly in each round.

The main contribution of this paper is an efficient online learning algorithm that effectively exploits particular structures of the logit model.  Based on the analytical properties of the logistic function, we first show that the linear regret defined in (\ref{eqn:reg}) and the nonlinear regret in (\ref{eqn:log:reg}) only differs by a constant factor, and then focus on minimizing the former one due to its simplicity. Similar to previous studies~\citep{Bandit:suvery},  we follow the principle of ``optimism in face of uncertainty'' to deal with the exploration-exploitation dilemma. The basic idea is to maintain a confidence region for $\w_*$, and choose an estimate from the confidence region and an action so that the linear reward is maximized. Thus, the problem reduces to the construction of the confidence region from one-bit feedback that satisfies (\ref{eqn:obr:2}). Based on the exponential concavity of the logistic loss, we propose to use a variant of the online Newton step~\citep{ML:Hazan:2007} to find the center of the confidence region and derive its width by a rather technical analysis of the updating rule. Theoretical analysis shows that our algorithm achieves a regret bound of $\O(d\sqrt{T})$,\footnote{We use the $\O$ notation to hide constant factors as well as polylogarithmic factors in $d$ and $T$.} which matches the result for SLB~\citep{Linear:Bandit:08}. Furthermore, we provide several strategies to reduce the computational cost of the proposed algorithm.
\section{Related Work} \label{sec:rel}
The stochastic multi-armed bandits (MAB)~\citep{Robbins:52},  has become the canonical formalism for studying the problem of decision-making under uncertainty. A long line of successive problems have been extensively studied in statistics~\citep{Bandit:book} and computer science~\citep{Bandit:suvery}.

\subsection{Stochastic Multi-armed Bandits (MAB)}
In their seminal paper, \citet{Lai1985} establish an asymptotic lower bound of $O(K \log T)$  for the expected cumulative regret over $T$ periods, under the assumption that the expected rewards of the best and second best arms are well-separated. By making use of \emph{upper confidence bounds} (UCB), they further construct policies which achieve the lower bound asymptotically. However, this initial algorithm is quite involved, because the computation of UCB relies on the entire sequence of rewards obtained so far.  To address this limitation, \citet{MAB:Agr:95} introduces a family of simpler policies that only needs to calculate the sample mean of rewards, and the regret retains the optimal logarithmic behavior. A finite time analysis of stochastic MAB is conducted by \citet{Auer:2002:FAM}. In particular, they propose a UCB-type algorithm based on the Chernoff-Hoeffding bound, and demonstrate it achieves the optimal logarithmic regret uniformly over time.

\subsection{Stochastic Linear Bandits (SLB)}
SLB is first studied by \citet{Auer:UCB}, who considers the case $\D$ is finite. Although an elegant UCB-type algorithm
named LinRel is developed, he fails to bound its regret due to independence issues. Instead, he designs a complicated master algorithm which uses LinRel as a subroutine, and achieves a regret bound of $\O((\log|\D|)^{3/2}  \sqrt{T d} )$, where $|\D|$ is the number of feasible decisions. In a subsequent work, \citet{Linear:Bandit:08} generalize LinRel slightly so that it can be applied in settings where $\D$ may be infinite. They refer to the new algorithm as ConfidenceBall$_2$, and show it enjoys a bound of $\O(d \sqrt{T})$, which does not depend on the cardinality of $\D$. Later, \citet{NIPS2011:LSB} improve the theoretical analysis of ConfidenceBall$_2$ by employing tools from the self-normalized processes. Specifically, the worst case bound is improved by a logarithmic factor and the constant is improved.

\subsection{ConfidenceBall$_2$}
To facilitate comparisons, we give a brief description of the ConfidenceBall$_2$ algorithm~\citep{Linear:Bandit:08}. In each round, the algorithm maintains a confidence region $\C_t$ such that with a high probability $\w_* \in \C_t$. Then, the algorithm finds the greedy optimistic decision
\[
\x_t = \argmax_{\x \in \D} \max_{\w \in \C_t} \x^\top \w.
\]
After submitting $\x_t$ to the oracle, the algorithm receives $y_t$ that satisfies (\ref{eqn:obr:1}). Given the past action-feedback pairs $(\x_1,y_1), \ldots (\x_t,y_t)$, the confidence region $\C_{t+1}$ is constructed as follows. The center of $\C_{t+1}$ is found by minimizing the $\ell_2$-regularized square loss, i.e.,
\[
\w_{t+1} = \argmax_{\w}  \sum_{i=1}^t (\x_i^\top \w-y_i)^2 + \lambda \|\w\|_2^2.
\]
Notice that $\w_{t+1}$ can be computed efficiently in an online fashion. Let $A_{t+1} = \lambda I + \sum_{i=1}^t \x_i \x_i^\top$. Based on the self-normalized bound for vector-valued martingales~\citep{NIPS2011:LSB}, the width of $\C_{t+1}$ can be characterized by
\[
(\w - \w_{t+1} )^\top A_{t+1} ( \w - \w_{t+1} ) \leq \delta_{t+1}
\]
for some constant $\delta_{t+1} >0$. As can be seen, the above procedure for constructing the confidence region is specially designed for the observation model in (\ref{eqn:obr:1}), and thus cannot be applied to the model in (\ref{eqn:obr:2}).

\subsection{Generalized Linear Bandit (GLB)}
\citet{NIPS2010} extend SLB to the nonlinear case based on the Generalized Linear Model framework of statistics. In the so-called  GLB model, $y_t$ is assumed to satisfy $\E[y_t|\x_t]= \mu(\x_t^\top  \w_*)$ where $\mu: \R \mapsto \R$ is certain link function. The regret is also defined in terms of $\mu(\cdot)$ and given by
\begin{equation} \label{eqn:reg:2}
T \max_{\x \in \D} \mu(\x^\top\w_*) - \sum_{t=1}^T \mu(\x_t^\top \w_*).
\end{equation}
Note that by setting $\mu(x)=\exp(x)/[1+\exp(x)]$, the problem considered in this paper becomes a special case of GLB. A UCB-type algorithm has been proposed for GLB and also achieves a regret bound of $\O(d \sqrt{T})$. Different from ConfidenceBall$_2$ which constructs a confidence region in the parameter space, the algorithm of \citet{NIPS2010} operates only in the reward space. However, the space and time complexities of that algorithm in the $t$-th iteration are $O(t)$ and  $O(t+|\D|)$, respectively. The $O(t)$ factor comes from the fact it needs to store the past action-feedback pairs $(\x_1,y_1), \ldots (\x_{t-1},y_{t-1})$ and use all of them to estimate $\w_*$. The $O(|\D|)$ factor is due to the fact it needs to calculate an upper bound for each arm in order to decide the next action $\x_t$.

\subsection{Adversarial Setting}
All the results mentioned above are under the stochastic setting, where the reward of each arm is generated from a unknown but \emph{fixed} distribution. A more general setting is the adversarial case, in which the reward from each arm may change arbitrary~\citep{Bandit:suvery}. The most well-known method for the adversarial multi-armed bandits is the Exp3 algorithm, which achieves a regret bound of $\O(\sqrt{KT})$~\citep{Auer:2003:NMB}.  The problem of adversarial linear bandits has been extensively studied, and the start-of-the-art regret bound is $\O(poly(d)\sqrt{T})$~\citep{Price:Bandit,Competing:Dark,Minimax:Linear}. For more results, please refer to \citet{Bandit:suvery}, \citet{Com:Bandit} and references therein.

\subsection{Bandit Learning with One-bit Feedback}
There are several new variants of bandit learning that also rely on one one-bit feedback, such as multi-class bandits~\citep{Kakade:2008:EBA,ICML14:Multiclass} and $K$-armed dueling bandits~\citep{COLT:Duel:2009,ICML14:Duel:Card}. For example, in multi-class bandits, the feedback is whether the predicted label is correct or not, and in $K$-armed dueling bandits, the feedback is the comparison between the rewards from two arms. However, none of them are designed for online linear optimization.

\subsection{One-bit Compressive Sensing (CS)}
Finally, we would like to discuss one closely related work in signal processing---one-bit Compressive Sensing (CS)~\citep{1bit:Boufounos,OneBit:Plan:Robust}. One-bit CS aims to recover a sparse vectors $\w_*$ from a set of one-bit measurements $\{y_i\}$ where $y_i$ is generated from $\x_i^\top \w_*$ according to certain observation model such as (\ref{eqn:obr:2}).  The main difference is that one-bit CS is studied in batch setting with the goal to minimize the recovery error, while our problem is studied in online setting with the goal to minimize the regret.

\section{Online Learning for Logit Model (OL$^2$M)}
We first describe the proposed algorithm for online stochastic linear optimization given one-bit feedback, next compare it with existing methods, then state its theoretical guarantees, and finally discuss implementation issues.
\subsection{The Algorithm}
For a positive definite matrix $A\in\R^{d\times d}$, the weighted $\ell_2$-norm is defined by $\|\x\|_A^2=\x^\top A \x$. Without loss of generality, we assume the decision space $\D$ is contained in the unit ball, that is,
\begin{equation} \label{eqn:assump}
\|\x\|_2 \leq 1, \ \forall \x \in \D.
\end{equation}
We further assume the $\ell_2$-norm of $\w_*$ is upper bounded by some constant $R$, which is known to the learner. Our first observation is that the linear regret in (\ref{eqn:reg}) and the nonlinear regret in (\ref{eqn:log:reg}) only differs by a constant factor as indicated below.
\begin{lemma} \label{lem:eqiv:regret} We have
\begin{equation} \label{eqn:eqiv:regret}
  \frac{1}{2(1+\exp(R))} (\ref{eqn:reg}) \leq (\ref{eqn:log:reg}) \leq   \frac{1}{4} (\ref{eqn:reg})
\end{equation}
\end{lemma}
In the following, we will develop an efficient algorithm that minimizes the linear regret, which in turn  minimizes the nonlinear regret as well.

\begin{algorithm}[t]
\caption{Online Learning for Logit Model (OL$^2$M)}
\begin{algorithmic}[1]
\STATE {\bf Input:} Step Size $\eta$, Regularization Parameter $\lambda$
\STATE $Z_1=\lambda I$, $\w_1=0$
\FOR{$t=1,2,\ldots$} \label{stp:1}
\STATE
\[
(\x_t,\wh_t) = \argmax_{\x \in \D, \w \in \C_t} \x^\top \w
\]
\STATE Submit $\x_t$ and observe $y_t \in \{ \pm 1 \}$
\STATE Solve the optimization problem in (\ref{eqn:update}) to find $\w_{t+1}$
\ENDFOR
\end{algorithmic}
\label{alg:1}
\end{algorithm}

The algorithm is motivated as follows. Suppose actions $\x_1,\ldots,\x_t$ have been submitted to the oracle, and let $y_1,\ldots,y_t$ be the one-bit feedback from the oracle.
To approximate $\w_*$, the most straightforward way is to find the maximum likelihood estimator by solving the following logistic regression problem
\[
\min_{\|\w\|_2 \leq R} \frac{1}{t} \sum_{i=1}^{t} \log \left(1+\exp(-y_i \x_i^\top \w) \right).
\]
However, this approach does not scale well since it requires the leaner to store the entire learning history. Instead, we propose an online algorithm to find an approximate solution. The key observation is that the logistic loss
\[
f_t(\w)= \log \left( 1 + \exp(-y_t \x_t^\top \w )\right)
\]
is exponentially concave over bounded domain~\citep{LR:hazan}, which motives us to apply a variant of the online Newton step~\citep{ML:Hazan:2007}. Specifically, we propose to find an approximate solution $\w_{t+1}$ by solving the following problem
\begin{equation} \label{eqn:update}
\min_{\|\w\|_2 \leq R} \frac{\|\w-\w_t\|_{Z_{t+1}}^2}{2} + \eta (\w-\w_t)^\top \nabla f_t(\w_t)
\end{equation}
where $\eta>0$ is the step size,
\begin{equation}\label{eqn:Zt}
Z_{t+1} = Z_t + \frac{\eta \beta}{2} \x_t \x_t^\top,
\end{equation}
and $\beta$ is defined in (\ref{eqn:beta}). Although our updating rule is similar to the method in~\citep{ML:Hazan:2007}, there also exist some differences. As indicated by (\ref{eqn:Zt}), in our case $\x_t \x_t^\top$ is used to approximate the Hessian matrix, while in \citet{ML:Hazan:2007} $\nabla f_t(\w_t) [\nabla f_t(\w_t)^\top]$ is used.

After a theoretical analysis, we are able to show that with a high probability
\begin{equation} \label{eqn:ct}
\w_* \in \C_{t+1}=\left\{ \w: \|\w - \w_{t+1} \|_{Z_{t+1}} \leq \sqrt{\gamma_{t+1}} \right\}
\end{equation}
where the value of $\gamma_{t+1}$ is given in (\ref{eqn:delta_t}). Given the confidence region, we adopt the principle of ``optimism in face of uncertainty'', and the next action $\x_{t+1}$ is given by
\begin{equation} \label{eqn:x:t}
(\x_{t+1},\wh_{t+1}) = \argmax_{\x \in \D, \w \in \C_{t+1}} \x^\top \w.
\end{equation}

At the beginning, we set
\[
Z_1=\lambda I, \textrm{ and } \w_1=0.
\]
The above procedure is summarized in Algorithm~\ref{alg:1}, and is refer to as Online Learning for Logit Model (OL$^2$M).

Since both ConfidenceBall$_2$~\citep{Linear:Bandit:08} and our OL$^2$M are UCB-type algorithms, their overall frameworks are similar. The main difference lies in the construction of the confidence region and the related analysis. While ConfidenceBall$_2$ uses online least square to update the center of the confidence region, OL$^2$M resorts to online Newton step. Due to the difference in the updating rule and the observation model, the self-normalized bound for vector-valued martingales~\citep{NIPS2011:LSB} can not be applied here.

Although our observation model in (\ref{eqn:obr:2}) can be handled by the Generalized Linear Bandit (GLB)~\citep{NIPS2010}, this paper differs from GLB in the following aspects.
\begin{compactitem}
  \item To estimate $\w_*$, GLB needs to store the learning history and perform batch updating in each round. In contrast, the proposed OL$^2$M performs online updating.
  \item While GLB only considers a finite number of arms, we allow the number of arms to be infinite.
  \item Our algorithm follows the learning framework of SLB. Thus, existing techniques for speeding up SLB can also be used to accelerate our algorithm, which is discussed in Section~\ref{sec:acce}.
\end{compactitem}
\subsection{Theoretical Guarantees}
The main theoretical contribution of this paper is the following theorem regarding the confidence region of $\w_*$ at each round.
\begin{thm} \label{thm:confidence}
With a probability at least $1-\delta$, we have
\[
\|\w_{t+1} - \w_*\|_{Z_{t+1}}  \leq \sqrt{\gamma_{t+1}},  \ \forall t>0
\]
where
\begin{equation} \label{eqn:delta_t}
\gamma_{t+1}=2 \eta \left[4R+  \left( \frac{4}{\beta}  + \frac{8}{3} R \right) \tau_t +  \frac{1}{ \beta}  \log \frac{\det(Z_{t+1})}{\det(Z_1)}\right] + \max \left(\lambda , \frac{\eta \beta}{2} \right) R^2,
\end{equation}
\begin{align}
\tau_t=&\log \left( \frac{2\lceil 2\log_2 t\rceil t^2}{\delta}\right) \label{eqn:tau:t}, \\
\beta=&\frac{1}{2(1+\exp(R))}. \label{eqn:beta}
\end{align}
\end{thm}
The main idea is to analyze the growth of $\|\w_{t+1} - \w_*\|_{Z_{t+1}}^2$ by exploring the properties of the logistic loss (Lemmas~\ref{lem:exp} and \ref{lemma:ft:pre}) and concentration inequalities for martingales (Lemma~\ref{lem:martin}). By a simple upper bound of  $\log \det(Z_{t+1})/\det(Z_1)$, we can show that the width of the confidence region is $O(\sqrt{d \log t})$.
\begin{cor} \label{cor:alpha:order} We have
\[
\log \frac{\det(Z_{t+1})}{\det(Z_1)} \leq d \log \left(1+ \frac{\eta \beta t}{2\lambda d}\right)
\]
and thus
\[
\gamma_{t+1} \leq O(d \log t),  \ \forall t>0.
\]
\end{cor}

Based on Theorem~\ref{thm:confidence}, we have the following regret bound for OL$^2$M.
\begin{thm} \label{thm:bound} With a probability at least $1-\delta$, we have
\[
T \max_{\x \in \D} \x^\top\w_* - \sum_{t=1}^T \x_t^\top \w_* \leq 4\sqrt{\frac{\gamma_T  T}{\eta \beta} \log\frac{\det(Z_{T+1})}{\det(Z_1)}}
\]
holds for all $T >0$.
\end{thm}
Combining with the upper bound in Corollary~\ref{cor:alpha:order}, the above theorem implies our algorithm achieves a regret bound of $\O(d\sqrt{T})$ which matches the bound for Stochastic Linear Bandits~\citep{Linear:Bandit:08}.
\subsection{Implementation Issues} \label{sec:acce}
The main computational cost of OL$^2$M comes from (\ref{eqn:x:t}) which is NP-hard in general~\citep{Linear:Bandit:08}. In the following, we discuss several strategies for reducing the computational cost.

\paragraph{Optimization Over Ball} As mentioned by \citet{Linear:Bandit:08}, in the special case that $\D$ is the unit ball, (\ref{eqn:x:t}) could be solved in time $O(poly(d))$. Here, we provide an explanation using techniques from convex optimization. To this end, we rewrite the optimization problem in (\ref{eqn:x:t}) as follows
\[
 \max_{\|\x\|_2 \leq 1, \|\w - \w_{t+1} \|_{Z_{t+1}} \leq \sqrt{\gamma_{t+1}}} \x^\top \w = \max_{\|\w - \w_{t+1} \|_{Z_{t+1}} \leq \sqrt{\gamma_{t+1}}}\|\w\|_2 \\
\]
which is equivalent to
\[
\min_{\|\w - \w_{t+1} \|_{Z_{t+1}}^2 \leq \gamma_{t+1}} -\|\w\|_2^2.
\]
The above problem is an optimization problem with a quadratic objective and one quadratic inequality constraint, it is well-known that strong duality holds provided there exists a strictly feasible point~\citep{Convex-Optimization}. Thus, we can solve its dual problem which is convex and given by
\[
\begin{array}{ll}
\max & \gamma \\
\st & \lambda \geq 0\\
& \left[ \begin{array}{cc}
- I +\lambda Z_{t+1} & -\lambda Z_{t+1} \w_{t+1} \\
-\lambda \w_{t+1}^\top Z_{t+1}  & \lambda (\|\w_{t+1}\|_{Z_{t+1}}^2 -\gamma_{t+1}) -\gamma \\
\end{array} \right] \succeq 0 \\
\end{array}
\]
After obtaining the dual solution, we can get the primal solution based on KKT conditions.

\paragraph{Enlarging the Confidence region} For a positive definite matrix $A\in\R^{d\times d}$, we define
\[
 \|\x\|_{1,A} = \|A^{1/2}\x\|_1.
\]
When studying SLB, \citet{Linear:Bandit:08} propose to enlarge the confidence region from $\C_{t+1}=\left\{ \w: \|\w - \w_{t+1} \|_{Z_{t+1}} \leq \sqrt{\gamma_{t+1}} \right\}$ to $\widetilde{\C}_{t+1}=\left\{ \w: \|\w - \w_{t+1} \|_{1,Z_{t+1}} \leq \sqrt{d  \gamma_{t+1}} \right\}$ such that the computational cost could be reduced.  This idea can be direct incorporated to our OL$^2$M. Let $\mathcal{E}_{t+1}$ be the set of extremal points of $\widetilde{\C}_{t+1}$. With this modification, (\ref{eqn:x:t}) becomes
\[
(\x_{t+1},\wh_{t+1}) = \argmax_{\x \in \D, \w \in \widetilde{\C}_{t+1}} \x^\top \w = \argmax_{\x \in \D, \w \in \mathcal{E}_{t+1}} \x^\top \w
\]
which means we just need to enumerate over the $2 d$ vertices in $\mathcal{E}_{t+1}$. Following the arguments in \citet{Linear:Bandit:08}, it is straightforward to show that the regret is only increased by a factor of $\sqrt{d}$.

\paragraph{Lazy Updating} \citet{NIPS2011:LSB} propose a lazy updating strategy which only needs to solve (\ref{eqn:x:t}) $O(\log T)$ times. The key idea is to recompute $\x_t$ whenever  $\det(Z_t)$ increases by a constant factor $(1+c)$. While the computation cost is saved dramatically, the regret is only increased by a constant factor $\sqrt{1+c}$. We provide the lazy updating version of OL$^2$M in Algorithm~\ref{alg:2}.

\begin{algorithm}[t]
\caption{OL$^2$M with Lazy Updating}
\begin{algorithmic}[1]
\STATE {\bf Input:} Step Size $\eta$, Regularization Parameter $\lambda$, Constant $c$
\STATE $Z_1=\lambda I$, $\w_1=0$, $\tau=1$
\FOR{$t=1,2,\ldots$}
\IF{$\det(Z_t) > (1+c) \det(V_\tau)$}
\STATE\[
(\x_t,\wh_t) = \argmax_{\x \in \D, \w \in \C_t} \x^\top \w
\]
\STATE $\tau=t$
\ENDIF
\STATE $\x_t=\x_\tau$
\STATE Submit $\x_t$ and observe $y_t \in \{ \pm 1 \}$
\STATE Solve the optimization problem in (\ref{eqn:update}) to find $\w_{t+1}$
\ENDFOR
\end{algorithmic}
\label{alg:2}
\end{algorithm}
\section{Analysis}
We here present the proofs of main theorems. The omitted proofs are provided in the appendix.

\subsection{Proof of Theorem~\ref{thm:confidence}} \label{sec:confi}
We begin with several lemmas that are central to our analysis.

Although the application of online Newton step~\citep{ML:Hazan:2007} in Algorithm~\ref{alg:1} is motivated from the fact that $f_t(\w)$ is exponentially concave over bounded domain, our analysis is built upon a related but different property that the logistic loss $\log(1+\exp(x))$ is strongly convex over bounded domain, from which we obtain the following lemma.
\begin{lemma} \label{lem:exp}  Denote the ball of radius $R$ by $\B_R$, i.e., $\B_R=\{ \w :\|\w\|_2 \leq R\}$. The following holds for $\beta \leq \frac{1}{2(1+\exp(R))}$:
\[
f_t(\w_2) \geq  f_t(\w_1) + [\nabla f_t(\w_1)]^\top (\w_2-\w_1) + \frac{\beta}{2} \left((\w_2-\w_1)^\top \x_t \right)^2, \  \forall \w_1, \w_2 \in \B_R.
\]
\end{lemma}
Comparing Lemma~\ref{lem:exp} with Lemma 3 in~\citep{ML:Hazan:2007}, we can see that the quadratic term in our inequality does not depends on $y_t$. This independence allows us to simplify the subsequent analysis involving martingales.

Our second lemma is devoted to analyzing the property of the updating rule in (\ref{eqn:update}).
\begin{lemma} \label{lem:updating:rule}
\begin{equation}\label{eqn:update:prep}
\langle \w_t - \w_*, \nabla f_t(\w_t) \rangle \leq  \frac{\|\w_t - \w_*\|_{Z_{t+1}}^2}{2\eta}-\frac{ \|\w_{t+1} - \w_*\|_{Z_{t+1}}^2}{2\eta}+ \frac{\eta}{2} \|\nabla f_t(\w_t)\|_{Z_{t+1}^{-1}}^2.
\end{equation}
\end{lemma}

For each function $f_t(\cdot)$, we denote its conditional expectation over $y_t$ by $\fb_t(\w)$, i.e.,
\begin{equation} \label{eqn:fb:def}
 \fb_t(\w) = \E_{y_t}\left[\log\left(1 + \exp\left( - y_t \x_t^{\top} \w\right)\right)\right].
\end{equation}
Based the property of Kullback--Leibler divergence~\citep{Elements:IT}, we obtain the following lemma.
\begin{lemma} \label{lemma:ft:pre}
We have
\[
 \fb_t(\w)  \geq  \fb_t(\w_*), \ \forall \w \in \R^d.
\]
\end{lemma}

Next, we introduce one inequality for bounding the weighted $\ell_2$-norm of the gradient
\begin{equation}\label{eqn:bound:grad}
\|\nabla f_t(\w)\|_{A}^2 = \left(\frac{\exp(-y_t \x_t^\top \w )}{1 + \exp(-y_t \x_t^\top \w )} \right)^2 \x_t^\top A \x_t  \leq  \|\x_t\|_A^2, \ \forall A \succeq 0, \ \w \in \R^d.
\end{equation}

We continue the proof of Theorem~\ref{thm:confidence} in the following. Our updating rule in (\ref{eqn:update}) ensures $\|\w_t\|_2 \leq R$, $\forall t>0$. Combining with the assumption $\|\w_*\|_2 \leq R$, Lemma~\ref{lem:exp} implies
\begin{equation} \label{eqn:1}
f_t(\w_t) \leq   f_t(\w_*) + [\nabla f_t(\w_t)]^\top (\w_t-\w_*)  - \frac{\beta}{2} \left((\w_*-\w_t)^\top \x_t\right)^2.
\end{equation}
By taking expectation over $y_t$, (\ref{eqn:1}) becomes
\[
\fb_t(\w_t) \leq    \fb_t(\w_*) + [\nabla \fb_t(\w_t)]^\top (\w_t-\w_*)  - \frac{\beta}{2} \left[\left((\w_*-\w_t)^\top \x_t\right)^2\right].
\]

Combining with Lemma~\ref{lemma:ft:pre}, we have
\[
\begin{split}
0 \leq & [\nabla \fb_t(\w_t)]^\top (\w_t-\w_*) - \frac{\beta}{2} \underbrace{\left((\w_*-\w_t)^\top \x_t\right)^2}_{:=a_t}\\
= & [\nabla f_t(\w_t)]^\top (\w_t-\w_*) -\frac{\beta}{2} a_t + \underbrace{[\nabla \fb_t(\w_t)- \nabla f_t(\w_t)]^\top (\w_t-\w_*)}_{:=b_t}\\
= & [\nabla f_t(\w_t)]^\top (\w_t-\w_*) - \frac{\|\w_t - \w_*\|_{Z_{t+1}}^2}{2\eta}+  \frac{\|\w_t - \w_*\|_{Z_{t+1}}^2}{2\eta} -\frac{\beta}{2}  a_t+b_t\\
\overset{\text{(\ref{eqn:update:prep})}}{\leq} & - \frac{\|\w_{t+1} - \w_*\|_{Z_{t+1}}^2}{2\eta}+ \frac{\eta}{2} \|\nabla f_t(\w_t)\|_{Z_{t+1}^{-1}}^2+  \frac{\|\w_t - \w_*\|_{Z_{t+1}}^2}{2\eta} -\frac{\beta}{2}  a_t+b_t\\
 \overset{\text{(\ref{eqn:bound:grad})}}{\leq}& - \frac{\|\w_{t+1} - \w_*\|_{Z_{t+1}}^2}{2\eta}+ \frac{\eta}{2}\underbrace{\|\x_t\|_{Z_{t+1}^{-1}}^2}_{:=c_t}+  \frac{\|\w_t - \w_*\|_{Z_{t+1}}^2}{2\eta} -\frac{\beta}{2}  a_t+b_t\\
\overset{\text{(\ref{eqn:Zt})}}{=}& - \frac{\|\w_{t+1} - \w_*\|_{Z_{t+1}}^2}{2\eta}-\frac{\beta}{2}  a_t+b_t +\frac{\eta}{2} c_t +  \frac{\|\w_t - \w_*\|_{Z_{t}}^2}{2\eta} + \frac{ \beta}{4} \left(\x_t^\top (\w_t - \w_*)\right)^2  \\
= &- \frac{\|\w_{t+1} - \w_*\|_{Z_{t+1}}^2}{2\eta}-\frac{\beta}{4}  a_t+b_t +\frac{\eta}{2} c_t +  \frac{\|\w_t - \w_*\|_{Z_{t}}^2}{2\eta}.
\end{split}
\]

We thus have
\[
\|\w_{t+1} - \w_*\|_{Z_{t+1}}^2 \leq  \|\w_t - \w_*\|_{Z_{t}}^2 - \frac{ \eta \beta}{2} a_t+ 2\eta b_t + \eta^2 c_t
\]
Summing the above inequality over iterations $1$ to $t$, we obtain
\begin{equation} \label{eqn:bound:1}
 \|\w_{t+1} - \w_*\|_{Z_{t+1}}^2 + \frac{ \eta \beta}{2} \sum_{i=1}^t a_i \leq \lambda R^2 + 2 \eta \sum_{i=1}^t  b_i + \eta^2 \sum_{i=1}^t  c_i.
\end{equation}
Next, we discuss how to bound the summation of martingale difference sequence $\sum_{i=1}^t  b_i$. To this end, we prove the following lemma, which is built up the Bernstein's inequality for martingales~\citep{bianchi-2006-prediction} and the peeling technique~\citep{Local_RC}.
\begin{lemma} \label{lem:martin} With a probability at least $1-\delta$,  we have
\[
\sum_{i=1}^t  b_i \leq 4R+ 2\sqrt{\tau_t\sum_{i=1}^t a_i} + \frac{8}{3} R \tau_t,  \ \forall t>0
\]
where $\tau_t$ is defined in (\ref{eqn:tau:t}).
\end{lemma}
From Lemma~\ref{lem:martin} and the basic inequality
\[
2\sqrt{\tau_t\sum_{i=1}^t a_i} \leq  \frac{\beta}{4}\sum_{i=1}^t a_i +  \frac{4}{\beta}\tau_t,
\]
with a probability at least $1-\delta$, we have
\begin{equation} \label{eqn:bound:2}
\sum_{i=1}^t  b_i \leq 4R+  \frac{\beta}{4}\sum_{i=1}^t a_i + \left( \frac{4}{\beta}  + \frac{8}{3} R \right) \tau_t
\end{equation}
holds for all $t>0$. Substituting (\ref{eqn:bound:2}) into (\ref{eqn:bound:1}), we obtain
\begin{equation} \label{eqn:bound:3}
 \|\w_{t+1} - \w_*\|_{Z_{t+1}}^2  \leq  \lambda R^2 + 2 \eta \left[4R+  \left( \frac{4}{\beta}  + \frac{8}{3} R \right) \tau_t\right] + \eta^2 \sum_{i=1}^t  c_i.
\end{equation}

Finally, we show an upper bound for $\sum_{i=1}^t  c_i$, which is a direct consequence of Lemma 12 in~\citet{ML:Hazan:2007}.
\begin{lemma} \label{lem:trivial} We have
\[
\sum_{i=1}^t \|\x_i\|_{Z_{i+1}^{-1}}^2 \leq \frac{2}{\eta \beta}  \log \frac{\det(Z_{t+1})}{\det(Z_1)}.
\]
\end{lemma}
We complete the proof by combining (\ref{eqn:bound:3}) with the above lemma.
\subsection{Proof of Lemma~\ref{lem:exp}}
We first show that the one-dimensional logistic loss $\ell(x)=\log(1+\exp(-x))$ is $\frac{1}{2(1+\exp(R))}$-strongly convex over domain $[-R,R]$.
It is easy to verify that $\forall x \in [-R, R]$,
\[
\ell''(x)= \frac{\exp(x)}{(1+\exp(x))^2}  \geq \frac{1}{2(1+\exp(R))}
\]
implying the strongly convexity of $\ell(\cdot)$. From the property of strongly convex, for any $a,  b \in [-R,R]$ we have
\begin{equation} \label{eqn:strongly:convex}
\ell(b) \geq \ell(a) + \ell'(a) (b-a) + \frac{\beta}{2} (b-a)^2.
\end{equation}

Notice that for any $\w_1, \w_2 \in \B_R$, we have
\[
y_t \x_t^\top \w_1, \ y_t \x_t^\top \w_2 \in [-R,R],
\]
since $y_t \in\{\pm1\}$ and $\|\x_t\|_2 \leq 1$. Substituting $a=y_t \x_t^\top \w_1$ and $b=y_t \x_t^\top \w_2$ into (\ref{eqn:strongly:convex}), we have
\[
\ell(y_t \x_t^\top \w_2) \geq \ell(y_t \x_t^\top \w_1)+ \frac{\beta}{2} (y_t \x_t^\top \w_2-y_t \x_t^\top \w_1)^2 + \ell'(y_t \x_t^\top \w_1) (y_t \x_t^\top \w_2-y_t \x_t^\top \w_1).
\]
We complete the proof by noticing
\[
f_t(\w_1)=\ell(y_t \x_t^\top \w_1), \ f_t(\w_2)=\ell(y_t \x_t^\top \w_2), \textrm{ and }\nabla f_t(\w_1)= \ell'(y_t \x_t^\top \w_1) y_t \x_t.
\]
\subsection{Proof of Lemma~\ref{lem:updating:rule}}
Lemma~\ref{lem:updating:rule} follows from a more general result stated below.
\begin{lemma} \label{lem:update} Let $M$ be a positive definite matrix, and
\[
\y = \mathop{\arg\min}\limits_{\w \in \W} \eta \langle \w, \g \rangle + \frac{1}{2}\|\w - \x\|_{M}^2,
\]
where $\W$ is a convex set. Then for all $\w \in \W$, we have
\[
\langle \x - \w, \g \rangle
\leq  \frac{\|\x - \w\|_{M}^2 - \|\y - \w\|_{M}^2}{2\eta} + \frac{\eta}{2} \|\g\|_{M^{-1}}^2.
\]
\end{lemma}
\begin{proof}
Since $\y$ is the optimal solution to the optimization problem, from the first-order optimality condition~\citep{Convex-Optimization}, we have
\begin{equation} \label{eqn:optimal}
\langle \eta  \g  + M (\y - \x),  \w-\y \rangle \geq 0, \ \forall \w \in \W.
\end{equation}

Based on the above inequality, we have
\[
\begin{split}
  & \|\x- \w\|_{M}^2 - \|\y - \w\|_{M}^2 \\
  = & \x^{\top} M \x - \y^{\top} M \y + 2 \langle  M (\y-\x),  \w \rangle  \\
  \overset{\text{(\ref{eqn:optimal})}}{\geq} & \x^{\top} M \x - \y^{\top} M \y+ 2 \langle M(\y - \x), \y \rangle - 2 \langle \eta  \g,  \w-\y  \rangle \\
  = & \|\y - \x\|_{M}^2  +  2 \langle \eta  \g,  \y-\x+ \x -\w \rangle  \\
   = & 2 \langle \eta  \g, \x -\w \rangle + \|\y - \x\|_{M}^2  +  2 \langle\eta  \g,  \y-\x \rangle  \\
\end{split}
\]
Combining with the following inequality
\[
 \|\y - \x\|_{M}^2  +  2 \langle\eta  \g,  \y-\x \rangle \geq   \min_{\w} \|\w\|_{M}^2+ 2 \langle \eta  \g,  \w \rangle  =- \eta^2  \|\g\|_{M^{-1}}^2,
\]
we have
\[
 \|\x- \w\|_{M}^2 - \|\y - \w\|_{M}^2   \geq  2 \langle \eta  \g, \x -\w \rangle  - \eta^2  \|\g\|_{M^{-1}}^2.
 \]
\end{proof}

\subsection{Proof of Lemma~\ref{lemma:ft:pre}}
For each $\w \in \R^d$, we introduce a discrete probability distribution $p_\w$ over $\{\pm 1\}$ such that
\[
p_\w(i)= \frac{1}{1+\exp(-i \x_t^\top \w)}, \ i \in \{\pm 1\}.
\]
Then, it is easy to verify that
\[
\fb_t(\w) =  - \sum_{i \in \{\pm 1\} } p_{\w_*}(i) \log p_\w(i) .
\]
As a result
\[
\begin{split}
  & \fb_t(\w) -  \fb_t(\w_*) \\
 = & \sum_{i \in \{\pm 1\} } p_{\w_*}(i) \log p_{\w_*}(i) - \sum_{i \in \{\pm 1\} } p_{\w_*}(i) \log p_\w(i) \\
 =& \sum_{i \in \{\pm 1\} }  p_{\w_*}(i) \log \frac{  p_{\w_*}(i)}{ p_{\w}(i)} = D_{KL}(p_{\w_*}\|p_\w) \geq 0
\end{split}
\]
where $D_{KL}(\cdot \| \cdot)$ is the Kullback--Leibler divergence between two distributions~\citep{Elements:IT}.
\subsection{Proof of Lemma~\ref{lem:martin}}
We need the Bernstein's inequality for martingales~\citep{bianchi-2006-prediction}, which is provided in Appendix~\ref{sec:bernstein}. Form our definition of $\fb_i(\cdot)$ in (\ref{eqn:fb:def}), it is clear
\[
b_i=[\nabla \fb_i(\w_i)- \nabla f_i(\w_i)]^\top (\w_i-\w_*)
\]
is a  martingale difference sequence. Furthermore,
\[
|b_i| \leq  \left |[\nabla \fb_i(\w_i)]^\top (\w_i-\w_*)\right|  + \left|[\nabla f_i(\w_i)]^\top (\w_i-\w_*)\right|  \leq 2 | \x_i^\top (\w_i-\w_*)|\leq  2 \|\w_i-\w_*\|_2 \leq 4R.
\]
Define the martingale $B_t= \sum_{i=1}^{t} b_i$. Define the conditional variance $\Sigma_t^2$ as
\[
\begin{split}
\Sigma_t^2 =& \sum_{i=1}^t \E_{y_i}   \left[\left( [\nabla \fb_i(\w_i)- \nabla f_i(\w_i)]^\top (\w_i-\w_*) \right)^2 \right]\\
 \leq & \sum_{i=1}^t \E_{y_i} \left[ \left(  \nabla f_i(\w_i)^\top (\w_i-\w_*) \right)^2 \right] \leq  \underbrace{\sum_{i=1}^t \left(\x_i^\top (\w_i-\w_*) \right)^2}_{:=A_t},
\end{split}
\]
where the first inequality is due to the fact that $\E[(\xi-\E[\xi])^2 ]\leq \E[\xi^2]$ for any random variable $\xi$.

In the following, we consider two different scenarios, i.e., $A_t \leq \frac{4R^2}{t}$ and $A_t > \frac{4R^2}{t}$.
\paragraph{$A_t \leq \frac{4R^2}{t}$} In this case, we have
\begin{equation} \label{eqn:case:1}
B_t \leq \sum_{i=1}^t |b_i| \leq 2 \sum_{i=1}^t| \x_i^\top (\w_i-\w_*)|\leq  2 \sqrt{t \sum_{i=1}^t \left(\x_i^\top (\w_i-\w_*)\right)^2} \leq 4R.
\end{equation}

\paragraph{$A_t > \frac{4R^2}{t}$}
Since $A_t$ in the upper bound for $\Sigma_t^2$ is a random variable, we cannot apply Bernstein's inequality directly. To address this issue, we make use of the peeling process  \citep{Local_RC}. Note that we have both a lower bound and an upper bound for $A_t$, i.e., $4R^2/t <  A_t \leq 4R^2 t$. Then,
\[
\begin{split}
& \Pr\left[B_t \geq 2\sqrt{A_t \tau_t} + \frac{8}{3} R \tau_t\right] \\
= & \Pr\left[B_t\geq 2\sqrt{A_t \tau_t} + \frac{8}{3} R\tau_t,  \frac{4R^2}{t} <  A_t \leq 4R^2 t\right] \\
=  & \Pr\left[B_t\geq 2\sqrt{A_t \tau_t} +  \frac{8}{3} R\tau_t, \Sigma_t^2 \leq A_t,  \frac{4R^2}{t} <  A_t \leq 4R^2 t \right] \\
\leq & \sum_{i=1}^m \Pr\left[ B_t\geq 2\sqrt{A_t \tau_t} + \frac{8}{3} R \tau_t, \Sigma_t^2 \leq A_t,  \frac{4R^22^{i-1}}{t} < A_t  \leq \frac{4R^22^{i}}{t}   \right] \\
\leq  &  \sum_{i=1}^m \Pr\left[B_t \geq \sqrt{2\frac{4R^2 2^{i}}{t}\tau_t} + \frac{8}{3} R \tau_t, \Sigma_t^2 \leq \frac{4R^2 2^i}{t} \right]  \leq   \ m e^{-\tau_t},
\end{split}
\]
where $m = \lceil 2\log_2 t\rceil$, and the last step follows the Bernstein's inequality for martingales. By setting $\tau_t= \log \frac{2m t^2}{\delta}$, with a probability at least $1-\delta/[2t^2]$, we have
\begin{equation} \label{eqn:case:2}
B_t \leq 2\sqrt{A_t \tau_t} + \frac{8}{3} R \tau_t.
\end{equation}

Combining (\ref{eqn:case:1}) and (\ref{eqn:case:2}), with a probability at least $1-\delta/[2t^2]$, we have
\[
B_t \leq 4R+ 2\sqrt{A_t \tau_t} + \frac{8}{3} R \tau_t.
\]
We complete the proof by taking the union bound over $t > 0$, and using the well-known result
\[
\sum_{t=1}^\infty \frac{1}{t^2} = \frac{\pi^2}{6} \leq 2.
\]

\subsection{Proof of Theorem~\ref{thm:bound}}
The proof is standard and can be found from~\citet{Linear:Bandit:08} and \citet{NIPS2011:LSB}. We include it for the sake of completeness.

Let $\x_* =\argmax_{\x \in \D} \x^\top\w_*$. Recall that in each round, we have
\[
(\x_t,\wh_t) = \argmax_{\x \in \D, \w \in \C_t} \x^\top \w.
\]
We decompose the instantaneous regret at round $t$ as follows
\[
\begin{split}
& \x_*^\top \w_* -  \x_t^\top \w_* \\
\leq &  \x_t^\top \wh_t-  \x_t^\top \w_* = \x_t^\top (\wh_t- \w_t) + \x_t^\top (\w_t - \w_*) \\
\leq & \left(\|\wh_t- \w_t\|_{Z_t} + \|\w_t - \w_*\|_{Z_t} \right)\|\x_t\|_{Z_t^{-1}}\leq  2 \sqrt{\gamma_t} \|\x_t\|_{Z_t^{-1}}.
\end{split}
\]
On the other hand, we always have
\[
\x_*^\top \w_* -  \x_t^\top \w_* \leq \|\x_*-\x_t\|_2 \|\w_*\|_2 \leq 2R.
\]

From the definition in (\ref{eqn:delta_t}), we have $\sqrt{\frac{2}{\eta \beta}\gamma_T} \geq R$. Thus, the total regret can be upper bounded by
\[
\begin{split}
& T \max_{\x \in \D} \x^\top\w_* - \sum_{t=1}^T \x_t^\top \w_*\\
\leq &  2 \sum_{t=1}^T \min\left(\sqrt{\gamma_t} \|\x_t\|_{Z_t^{-1}}, R\right)\\
\leq & 2 \sqrt{\frac{2}{\eta \beta}\gamma_T} \sum_{t=1}^T \min\left(  \sqrt{\frac{\eta \beta}{2}} \|\x_t\|_{Z_t^{-1}}, 1\right)\\
\leq & 2 \sqrt{\frac{2 T}{\eta \beta}\gamma_T} \sqrt{\sum_{t=1}^T \min\left( \frac{\eta \beta}{2} \|\x_t\|_{Z_t^{-1}}^2, 1\right)}.
\end{split}
\]

To proceed, we need the following results from Lemma 11 in \citet{NIPS2011:LSB},
\[
 \sum_{t=1}^T \min\left(  \frac{\eta \beta}{2}\|\x_t\|_{Z_t^{-1}}^2, 1\right) \leq  2 \sum_{t=1}^T \log \left(1+ \frac{\eta \beta}{2}\|\x_t\|_{Z_t^{-1}}^2\right)
\]
and
\[
\begin{split}
&\det(Z_{T+1})=  \det\left(Z_T+ \frac{\eta \beta}{2} \x_T \x_T^\top\right) \\
=&\det(Z_T) \det\left(I+\frac{\eta \beta}{2} Z_T^{-1/2} \x_T \x_T^\top Z_T^{-1/2}\right) \\
=&\det(Z_T)\left(1+ \frac{\eta \beta}{2} \|\x_T\|_{Z_T^{-1}}^2 \right) = \det(Z_1)\prod_{t=1}^T\left(1+ \frac{\eta \beta}{2} \|\x_t\|_{Z_t^{-1}}^2 \right). \\
\end{split}
\]

Combining the above inequations, we have
\[
T \max_{\x \in \D} \x^\top\w_* - \sum_{t=1}^T \x_t^\top \w_* \leq 4\sqrt{\frac{\gamma_T  T}{\eta \beta} \log\frac{\det(Z_{T+1})}{\det(Z_1)}}.
\]
\section{Conclusions}
In this paper, we consider the problem of online linear optimization under one-bit feedback. Under the assumption that the binary feedback is generated from the logit model, we develop a variant of the online Newton step to approximate the unknown vector, and discuss how to construct the confidence region theoretically. Given the confidence region, we choose the action that produces maximal reward in each round. Theoretical analysis reveals that our algorithm achieves a regret bound of $\O(d\sqrt{T})$.

The current algorithm assumes that the one-bit feedback is generated from a logit model. In contrast, a much broader class of observation models are allowed in one-bit compressive sensing~\citep{OneBit:Plan:Robust}, as long as there is a positive correlation between the one-bit output and the real-valued measurement. In the future, we will investigate how to extend our algorithm to other observation models. Another direction is to consider the adversary setting where the unknown vector $\w_*$ may change from time to time.

\appendix
\section{Proof of Lemma~\ref{lem:eqiv:regret}}
Let $\mu(x)=\frac{\exp(x)}{1+\exp(\x)}$.  It is easy to verify that $\forall x \in [-R, R]$,
\begin{equation} \label{eqn:gradient:bound}
\frac{1}{2(1+\exp(R))} \leq \mu'(\x)= \frac{\exp(x)}{(1+\exp(x))^2}  \leq \frac{1}{4}
\end{equation}
Note that for any $-R \leq a \leq b \leq R$, we have
\begin{equation} \label{eqn:int}
\mu(b) = \mu(a) + \int_a^b \mu'(x) d x
\end{equation}
Combining (\ref{eqn:gradient:bound}) with (\ref{eqn:int}), we have
\[
 \frac{1}{2(1+\exp(R))} (b-a) \leq \mu(b) - \mu(a) \leq  \frac{1}{4} (b-a)
\]

Let
\[
\x_*=\argmax_{\x \in \D} \x^\top\w_* = \argmax_{\x \in \D} \frac{\exp(\x^\top \w_*)}{1+\exp(\x^\top \w_*)}
\]
Since $ -R \leq \x_t^\top \w_* \leq \x_*^\top \w_* \leq R$, we have
\[
 \frac{1}{2(1+\exp(R))} \left(\x_*^\top \w_*-\x_t^\top \w_* \right) \leq \frac{\exp(\x_*^\top \w_*)}{1+\exp(\x_*^\top \w_*)} - \frac{\exp( \x_t^\top \w_*)}{1+\exp( \x_t^\top \w_*)}  \leq \frac{1}{4} \left(\x_*^\top \w_*-\x_t^\top \w_* \right)
\]
which implies (\ref{eqn:eqiv:regret}).

\section{Proof of Lemma~\ref{lem:trivial}}
We have
\[
\|\x_i\|_{Z_{i+1}^{-1}}^2 =  \frac{2}{\eta \beta} \langle Z_{i+1}^{-1}, Z_{i+1}-Z_i \rangle \leq \frac{2}{\eta \beta} \log \frac{\det(Z_{i+1})}{\det(Z_i)},
\]
where the inequality follows from Lemma 12 in~\citet{ML:Hazan:2007}. Thus, we have
\[
\sum_{i=1}^t  \|\x_i\|_{Z_{i+1}^{-1}}^2 \leq  \frac{2}{\eta \beta} \sum_{i=1}^t  \log \frac{\det(Z_{i+1})}{\det(Z_i)}  =  \frac{2}{\eta \beta}  \log \frac{\det(Z_{t+1})}{\det(Z_1)}.
\]
\section{Proof of Corollary~\ref{cor:alpha:order}}
Recall that
\[
Z_{t+1} = Z_1 + \frac{\eta \beta}{2} \sum_{i=1}^t  \x_t \x_t^\top
\]
and $\|\x_t\|_2 \leq 1$ for all $t>0$. From Lemma 10 of~\citet{NIPS2011:LSB}, we have
\[
 \det(Z_{t+1}) \leq \left(\lambda  + \frac{\eta \beta t}{2 d}\right)^d.
\]
Since $\det(Z_1)=\lambda^d$, we have
\[
\log \frac{\det(Z_{t+1})}{\det(Z_1)} \leq d \log \left(1+ \frac{\eta \beta t}{2\lambda d} \right).
\]

\section{Bernstein's Inequality for Martingales} \label{sec:bernstein}
\begin{thm} \label{thm:bernstein}  Let $X_1, \ldots , X_n$ be a bounded martingale difference sequence with respect to the filtration $\F = (\F_i)_{1\leq i\leq n}$ and with $|X_i| \leq K$. Let
\[
S_i = \sum_{j=1}^i X_j
\]
be the associated martingale. Denote the sum of the conditional variances by
\[
    \Sigma_n^2 = \sum_{t=1}^n \E\left[X_t^2|\F_{t-1}\right].
\]
Then for all constants $t$, $\nu > 0$,
\[
\Pr\left[ \max\limits_{i=1,\ldots,n} S_i > t \mbox{ and } \Sigma_n^2 \leq \nu \right] \leq \exp\left(-\frac{t^2}{2(\nu + Kt/3)} \right),
\]
and therefore,
\[
    \Pr\left[ \max\limits_{i=1,\ldots,n} S_i > \sqrt{2\nu t} + \frac{2}{3}Kt \mbox{ and } \Sigma_n^2 \leq \nu \right] \leq e^{-t}.
\]
\end{thm}




\vskip 0.2in
\bibliography{E:/MyPaper/ref}
\end{document}